\documentclass[wcp]{jmlr}
\ifdefined\arxivmode
\usepackage{amsthm}
\usepackage[numbers]{natbib} 
\usepackage[body={7in, 9in},left=1in,right=1in]{geometry}
\else
\fi
\usepackage{graphicx}
\usepackage{amsbsy}
\usepackage{amssymb}
\usepackage{amsmath}
\usepackage{amsfonts}
\usepackage{color}
\usepackage{latexsym}
\usepackage{epstopdf}
\usepackage{hyperref}
\usepackage{array}
\usepackage{enumerate}
\usepackage{xspace} 
\usepackage{algorithm}
\usepackage{algorithmic}
\usepackage{enumitem} 
\ifdefined\arxivmode
\else
\def\ispres{1}
\fi

\def\balign#1\ealign{\begin{align}#1\end{align}}
\def\baligns#1\ealigns{\begin{align*}#1\end{align*}}
\def\bitemize#1\eitemize{\begin{itemize}#1\end{itemize}}
\def\benumerate#1\eenumerate{\begin{enumerate}#1\end{enumerate}}

\newcommand{\qtext}[1]{\quad\text{#1}\quad} 

\let\originalleft\left
\let\originalright\right
\renewcommand{\left}{\mathopen{}\mathclose\bgroup\originalleft}
\renewcommand{\right}{\aftergroup\egroup\originalright}

\def\tinycitep*#1{{\tiny\citep*{#1}}}
\def\tinycitealt*#1{{\tiny\citealt*{#1}}}
\def\tinycite*#1{{\tiny\cite*{#1}}}
\def\smallcitep*#1{{\scriptsize\citep*{#1}}}
\def\smallcitealt*#1{{\scriptsize\citealt*{#1}}}
\def\smallcite*#1{{\scriptsize\cite*{#1}}}

\def\reals{\mathbb{R}} 

\def\<{\left\langle} 
\def\>{\right\rangle}

\def\defeq{\triangleq} 

\def\indic#1{\mbf{1}\left\{{#1}\right\}} 
\providecommand{\argmin}{\mathop\mathrm{arg min}}

\providecommand{\dom}{\mathop\mathrm{dom}}

\newcommand{\algref}[1]{Algorithm~{\ref{alg:#1}}}

\newcommand{\figref}[1]{Figure~{\ref{fig:#1}}}
\newcommand{\lemref}[1]{Lemma~{\ref{lem:#1}}}

\newcommand{\secref}[1]{Section~{\ref{sec:#1}}}

\newcommand{\propref}[1]{Proposition~{\ref{prop:#1}}}

\ifdefined\ispres\else
\newtheorem{theorem}{Theorem}
\newtheorem{lemma}[theorem]{Lemma}

\renewenvironment{proof}{\noindent\textbf{Proof}\hspace*{1em}}{\qed\\}
\newenvironment{proof-sketch}{\noindent\textbf{Proof Sketch}
  \hspace*{1em}}{\qed\bigskip\\}
\newenvironment{proof-idea}{\noindent\textbf{Proof Idea}
  \hspace*{1em}}{\qed\bigskip\\}
\newenvironment{proof-of-lemma}[1][{}]{\noindent\textbf{Proof of Lemma {#1}}
  \hspace*{1em}}{\qed\\}
\newenvironment{proof-of-theorem}[1][{}]{\noindent\textbf{Proof of Theorem {#1}}
  \hspace*{1em}}{\qed\\}
\newenvironment{proof-attempt}{\noindent\textbf{Proof Attempt}
  \hspace*{1em}}{\qed\bigskip\\}

\newtheorem{proposition}[theorem]{Proposition}

\fi

\newcommand{\breg}{b_{\textrm{reg}}}
\newcommand{\dataset}{\mathcal{D}}

\def\indic#1{\mathbb{I}\left[{#1}\right]}
\newcommand{\AMS}{\mathrm{AMS}}
\ifdefined\arxivmode
\title{Weighted Classification Cascades for Optimizing Discovery Significance in the HiggsML Challenge}

\author{Lester Mackey \and Jordan Bryan \and Man Yue Mo}

\date{\today}
\fi

\begin{document}

\ifdefined\arxivmode
\else
\jmlrvolume{42}
\jmlryear{2015}
\jmlrworkshop{HEPML 2014}
\jmlrpages{129-134}
\title[Weighted Classification Cascades]{Weighted Classification Cascades for Optimizing Discovery Significance in the HiggsML Challenge}
\author{\Name{Lester Mackey}\Email{lmackey@stanford.edu}\\
\addr{Stanford University}
\AND
\Name{Jordan Bryan}\Email{jgbryan@stanford.edu} \\
\addr{Stanford University}
\AND
\Name{Man Yue Mo}\Email{manyue.mo@gmail.com}
}

\editor{Glen Cowan, C\'{e}cile Germain, Isabelle Guyon, Bal\`{a}zs K\'egl, David Rousseau}\fi

\maketitle

\begin{abstract}
	We introduce a minorization-maximization approach to optimizing common measures of discovery significance in 
	high energy physics.
	The approach alternates between solving a weighted binary classification problem and 
	updating class weights in a simple, closed-form manner.
	Moreover, an argument based on convex duality shows that an improvement in weighted classification error on any round yields a commensurate improvement in discovery significance.
	We complement our derivation with experimental results from the 2014 Higgs boson machine learning challenge.
\end{abstract}

\ifdefined\arxivmode
\else
\fi
\section{Weighted Classification Cascades for Optimizing AMS}
\label{sec:WCC}
This paper derives a minorization-maximization approach~\citep{Lange00} to optimizing common measures of discovery significance in high energy physics.
We begin by introducing notation adapted from the 2014 Higgs boson machine learning (HiggsML) challenge\footnote{Readers unfamiliar with the setting and motivation of the HiggsML challenge may wish to review the challenge documentation~{\citep{Adam-BourdariosCoGeGuKeRo14}} before proceeding.} \citep{Adam-BourdariosCoGeGuKeRo14}.
Let $\dataset = \{(x_1, y_1, w_1),\dots, (x_n, y_n, w_n)\}$ represent a weighted dataset with feature vectors $x_i \in \mathcal{X}$, labels $y_i \in \{-1,1 \}$, and weights $w_i > 0$, and let
	$g : \mathcal{X} \to \{-1,1 \}$ represent a classifier which assigns labels to each datapoint $x \in \mathcal{X}$.  Then we may define the weighted number of
\begin{itemize}[nosep]
	\item true positives produced by $g$ on $\dataset$, $\displaystyle s_{\dataset}(g) = \sum_{i=1}^n w_i \indic{g(x_i) = 1, y_i = 1}$;
	\item false positives\footnote{The quantity $b_{\dataset}(g)$ may also include a constant additive regularization term, such as the quantity $\breg$ described in the HiggsML challenge documentation~{\citep{Adam-BourdariosCoGeGuKeRo14}}.} produced by $g$ on $\dataset$, $\displaystyle b_{\dataset}(g) = \sum_{i=1}^n w_i \indic{g(x_i) = 1, y_i = -1}$;
	\item positives produced by $g$ on $\dataset$, $\displaystyle n_{\dataset}(g) = s_{\dataset}(g) + b_{\dataset}(g)$;
	\item positives in $\dataset$, $\displaystyle p_{\dataset} = \sum_{i=1}^n w_i \indic{y_i = 1}$;
	\item and false negatives produced by $g$ on $\dataset$, $\displaystyle \tilde{s}_{\dataset}(g) = p_{\dataset} - s_{\dataset}(g)$.
\end{itemize}

Our aim is to maximize the measures of \emph{approximate median significance} (AMS)~\citep{CowanCrGrVi11},
\begin{align*}
	\AMS_2(g,\dataset) &= \sqrt{2\,b_{\dataset}(g)f_2\left(\frac{s_{\dataset}(g)}{b_{\dataset}(g)}\right)}
	\qtext{for} f_2(t) = (1+t)\log(1+t) - t \qtext{and} \\
	\AMS_3(g,\dataset) &= \sqrt{2\,b_{\dataset}(g)f_3\left(\frac{s_{\dataset}(g)}{b_{\dataset}(g)}\right)}
	\qtext{for} f_3(t) = (1/2)t^2,
\end{align*}
which were employed as utility measures for the HiggsML challenge~\citep{Adam-BourdariosCoGeGuKeRo14}.
However, the approach we pursue applies equally to any utility measure of the form
\begin{align} \label{eqn:general-form}
	h\left(b_{\dataset}(g)f\left(\frac{s_{\dataset}(g)}{b_{\dataset}(g)}\right)\right)
\end{align}
where $h$ is increasing and $f$ is closed proper convex and differentiable.

We first observe that $f_2$ and $f_3$ are closed proper convex functions and hence may be rewritten in terms of their convex conjugates~\citep{BorweinLe10}.
The following \emph{linearization lemma} makes this more precise.
\begin{lemma}[Linearization Lemma]\label{lem:linearize}
Consider a differentiable, closed proper convex function $f:\reals\to\reals$ and real numbers $a >0$ and $c$ with $c/a$ in the effective domain of $f$.
If $f^*(u) \defeq \sup_{t \in \dom(f)} tu - f(t)$ is the convex conjugate of $f$, then
\begin{align} \label{eqn:fenchel-young}
{\textstyle a\,f\left(\frac{c}{a}\right)}
	= -\inf_{u\in \dom(f^*)} - c u + a f^*(u)
\end{align}
where the minimum on the right-hand side is achieved by $u^* = f'(c/a)$.
\end{lemma}
\begin{proof}
The representation \eqref{eqn:fenchel-young} is a direct application of the 
Fenchel-Young inequality~\citep{BorweinLe10}, which further implies that
${\textstyle a\,f\left(\frac{c}{a}\right)} \geq c f'(c/a) - a f^*(f'(c/a))$.
The convexity and differentiability of $f$ and the positivity of $a$ further imply that
${\textstyle a\,f\left(\frac{c}{a}\right)} \leq af(v) + a(c/a - v)f'(c/a)$ for all $v\in\dom(f)$.
Taking an infimum over $v\in\dom(f)$ on the righthand side yields 
${\textstyle a\,f\left(\frac{c}{a}\right)} = c f'(c/a) - a f^*(f'(c/a))$ as advertised.
\end{proof}

By applying this lemma to our expressions for $\AMS_2$ and $\AMS_3$, we obtain fruitful 
variational representations for our significance measures.
\begin{proposition}[Variational Representations for Approximate Median Significance]
\label{prop:AMS}
\begin{align*}
-\frac{1}{2}\AMS_2(g,\dataset)^2 
	&= \inf_{u} R_2(g,u,\dataset) 
	\qtext{for} R_2(g,u,\dataset) \defeq b_{\dataset}(g)\,(e^u - u - 1) + \tilde{s}_{\dataset}(g)\,u - p_{\dataset}\,u, \\ 
u_2^* 
	&\defeq \argmin_u R_2(g,u,\dataset) = \log(s_{\dataset}(g)/b_{\dataset}(g) + 1), \\
-\frac{1}{2}\AMS_3(g,\dataset)^2
	&= \inf_u R_3(g,u,\dataset)
	\qtext{for} R_3(g,u,\dataset) \defeq b_{\dataset}(g)\,u^2/2 + \tilde{s}_{\dataset}(g)\,u  - p_{\dataset}\,u, \qtext{and} \\
u_3^* 
	&\defeq \argmin_u R_3(g,u,\dataset) = s_{\dataset}(g)/b_{\dataset}(g).
\end{align*}
\end{proposition}
\begin{proof}
To obtain the result for $-\frac{1}{2}\AMS_m(g,\dataset)^2$ for $m\in\{2,3\}$ we apply \lemref{linearize} 
with $a = b_{\dataset}(g),$ $c = s_{\dataset}(g) = p_{\dataset}-\tilde{s}_{\dataset}(g),$ and $f = f_m$ noting that
$
	f_2^*(u) = e^u - u - 1,\ f_2'(t) = \log(t+1),\ 
	f_3^*(u) = u^2/2,\ \text{and } f_3'(t) = t. 
$
\end{proof}

\propref{AMS} shows that, for $m\in\{2,3\}$, maximizing $\AMS_m(g,\dataset)$ over $g$ is equivalent to minimizing $R_m(g,u,\dataset)$ jointly over $f$ and $u$.  
To minimize $R_m(g,u,\dataset)$, we adopt a coordinate descent strategy which alternates between optimizing $f$ with $u$ held fixed and updating $u$ with $f$ held fixed.
Optimizing $f$ for fixed $u$ is equivalent to solving a weighted binary classification problem with class weights determined by $u$.
Consequently, this step can be carried out using any classification procedure that supports observation weights.
Furthermore, we have seen that the optimal value $u^*$ for a given $f$ can be computed in closed form.
Thus, our proposed optimization scheme consists of solving a series of weighted binary classification problems, a \emph{weighted classification cascade}.
The cascade steps for optimizing $\AMS_2$ and $\AMS_3$ are presented in \algref{ams2} and \algref{ams3} respectively; 
an illustration of weighted classification cascade progress is provided in \figref{mm-progress}.

\begin{algorithm}[ht]
   \caption{Weighted Classification Cascade for $\AMS_2$}
   \label{alg:ams2}
\begin{algorithmic}
	 \STATE {\bfseries input:} $u_0 > 0$
	 \FOR{t = 1 \TO T }
	 \STATE $g_{t}\gets$ approximate minimizer of weighted classification error $b_{\dataset}(g)\,(e^{u_{t-1}} - u_{t-1} - 1) + \tilde{s}_{\dataset}(g)\,u_{t-1}$, obtained from any weighted classification procedure
	 \STATE $u_{t}\gets \log(s_{\dataset}(g_t)/b_{\dataset}(g_t) + 1)$
	 \ENDFOR
	 \RETURN $g_T$
\end{algorithmic}
\end{algorithm}
\begin{algorithm}[ht]
   \caption{Weighted Classification Cascade for $\AMS_3$}
   \label{alg:ams3}
\begin{algorithmic}
	 \STATE {\bfseries input:} $u_0 > 0$
	 \FOR{t = 1 \TO T }
	 \STATE $g_{t}\gets$ approximate minimizer of weighted classification error $b_{\dataset}(g)\,u_{t-1}^2/2 + \tilde{s}_{\dataset}(g)\,u_{t-1}$, obtained from any weighted classification procedure
	 \STATE $u_{t}\gets s_{\dataset}(g_t)/b_{\dataset}(g_t)$
	 \ENDFOR
	 \RETURN $g_T$
\end{algorithmic}
\end{algorithm}

Finally, we note that $\AMS_m$ is guaranteed to increase whenever a newly selected scoring function $g_{t+1}$ achieves smaller weighted classification error with respect to $u_t$ than its predecessor $g_t$, since in this case $R_m(g_{t+1},u_{t},\dataset)
	< R_m(g_{t},u_{t},\dataset)$, and hence
\[
-\frac{1}{2}\AMS_m(g_{t+1},\dataset)^2
	\leq R_m(g_{t+1},u_{t},\dataset)
	< R_m(g_{t},u_{t},\dataset)
	= -\frac{1}{2}\AMS_m(g_{t},\dataset)^2.
\]
Such a monotonicity property is characteristic of majorization-minimization and minorization-maximization algorithms~\citep{Lange00}.
\begin{figure}[h] \label{fig:mm-progress}
\centering
  \includegraphics[width=1\linewidth]{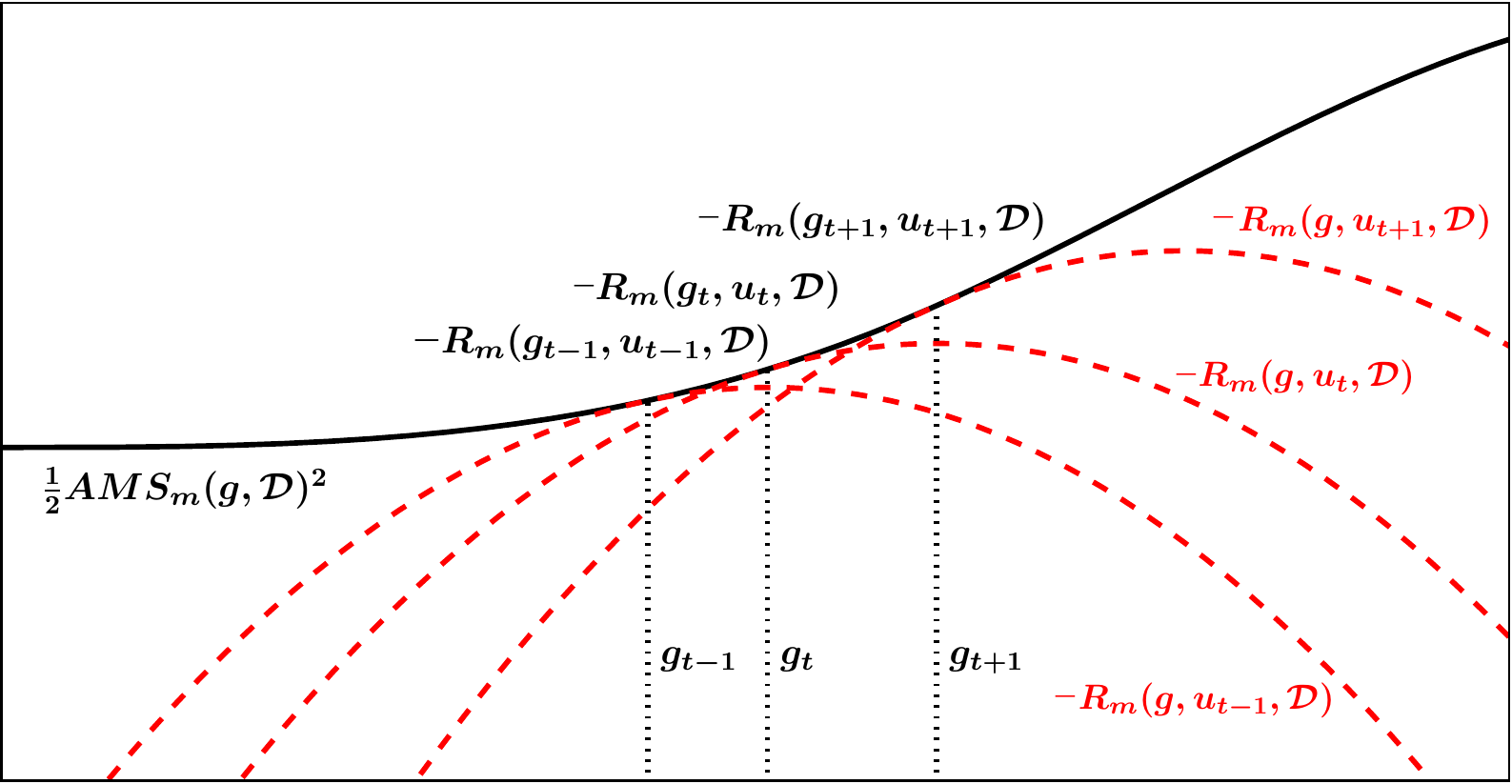}
  \caption{An illustration of the progress of a weighted classification cascade.}
\end{figure}

\subsection{Related work}
The functional form $b_{\dataset}(g)f\left(\frac{s_{\dataset}(g)}{b_{\dataset}(g)}\right)$ for convex $f$ is evocative of the class of discrepancy measures known as $f$-divergences~\citep{LieseVa06}.  Indeed, $b_{\dataset}(g)f\left(\frac{s_{\dataset}(g)}{b_{\dataset}(g)}\right)$ can be viewed as a generalized $f$-divergence between two unnormalized measures.
\citet{NguyenWaJo10} and \citet{Lexa12} have derived algorithms analogous to those derived here for optimizing $f$-divergences.

\section{HiggsML Challenge Case Study}
While the algorithms of \secref{WCC} provide simple recipes for turning any classifier that supports class weights into a training set AMS maximizer, the procedures should be coupled with effective regularization strategies to ensure adequate generalization from training error to held-out test error. 
In this section, we will describe the practical strategies employed by the HiggsML challenge team \texttt{mymo}, which incorporated two variants of weighted classification cascades into its final contest solution.

The first cascade variant used the XGBoost implementation of gradient tree boosting\footnote{\url{https://github.com/tqchen/xgboost}} to learn the base classifier $g_t$ on each round of \algref{ams2}.
To curb overfitting to the training set, on each cascade round, the team computed weighted true and false positive counts on a held-out validation dataset $\dataset_{\text{val}}$ and updated the class weight parameter $u_t$ using 
$s_{\dataset_{\text{val}}}(g_t)$ and $b_{\dataset_{\text{val}}}(g_t)$ in lieu of
$s_{\dataset}(g_t)$ and $b_{\dataset}(g_t)$.
The cascading procedure was run until the validation set AMS failed to increase (this often occurred on the third iteration) and was then run for a small number of additional rounds (typically ten).
Since XGBoost is a randomized learning algorithm, this entire cascade was rerun multiple times, and the classifiers from those cascade iterations yielding the highest validation set AMS scores were incorporated into the final solution ensemble.

The second cascade variant maintained a single persistent classifier, the complexity of which grew on each cascade round.  More precisely, team \texttt{mymo} developed a customized XGBoost classifier that, on cascade round $t$, introduced a single new decision tree based on the gradient of the round $t$ weighted classification error in \algref{ams2}.  In effect, each classifier $g_t$ was warm-started from the prior round classifier $g_{t-1}$.
For this variant, the number of cascade iterations $T$ was typically set to $500$.

The final contest solution was an ensemble of cascade procedures of each variant and several non-cascaded XGBoost, random forest, and neural network models.
The non-cascade models together yielded a private leaderboard score of 3.67 (198th place on the private leaderboard).
Incorporating the cascade models boosted that score to 3.72594, leaving team \texttt{mymo} in 31st place out of the 1785 teams in the competition. 
A separate post-challenge assessment by team \texttt{mymo} revealed that averaging the predictions of ten models, five standard XGBoost models trained without cascade weighting for $T=500$ iterations and five XGBoost models trained with the second variant
of cascade weighting for $T=500$ iterations led to a private leaderboard score of 3.72.
These results are evidence for the utility of cascading, and we hypothesize that additional benefits will be revealed by a more comprehensive empirical evaluation of cascade regularization strategies.

\bibliographystyle{abbrvnat}
\bibliography{refs}
\end{document}